\documentclass[12pt]{llncs}
\usepackage{amssymb,amsmath,graphicx}
\usepackage[top=3.5cm, bottom=3.5cm, left=3.5cm, right=3.5cm]{geometry}

\usepackage{url}
\usepackage{placeins}
\newcommand{\keywords}[1]{\par\addvspace\baselineskip
\noindent\keywordname\enspace\ignorespaces#1}

\begin{document}

\mainmatter  

\title{A Reverse Hex Solver}

\titlerunning{A Rex Solver}

\author{Kenny Young%
\thanks{The authors gratefully acknowledge the support of NSERC.}
\and Ryan B.\ Hayward}
\authorrunning{A Rex Solver}

\institute{Dept.\ of Computing Science, UAlberta, Canada,
hayward@ualberta.ca, \url{http://webdocs.cs.ualberta.ca/~hayward/}}

\maketitle

\begin{abstract}
We present Solrex, an automated solver for the game of Reverse Hex.
Reverse Hex, also known as Rex, or Mis\`{e}re Hex,
is the variant of the game of Hex in which the player who joins
her two sides loses the game.
Solrex performs a mini-max search of the state space using
Scalable Parallel Depth First Proof Number Search, 
enhanced by the pruning of inferior moves and 
the early detection of certain winning strategies.

Solrex is implemented on the same code base as the Hex program Solver,
and can solve arbitrary positions on board sizes up to 6$\times$6,
with the hardest position taking less than four hours on four
threads.
\end{abstract}

\section{Introduction}
In 1942 Piet Hein invented the two-player board game
now called Hex \cite{Hein42}.
The board is covered with a four-sided array of hexagonal cells.
Each player is assigned two opposite sides of the board.
Players move in alternating turns.
For each turn, a player places
one of their stones on an empty cell.
Whoever connects their two sides with a path of their stones is the winner.

In his 1957 {\em Scientific American} Mathematical Games column, 
{\em Concerning the game of Hex, which may be played
on the tiles of the bathroom floor},
Martin Gardner mentions the mis\`{e}re version of Hex
known as Reverse Hex, or Rex, or Mis\'{e}re Hex:
whoever joins their two sides {\it loses} \cite{Gard57a}.
See Figure~\ref{fig:1}.

\begin{figure}[htb]\centering
\includegraphics[scale=1.2]{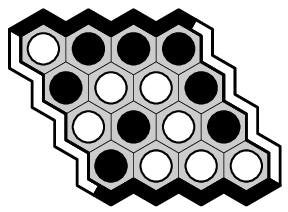}
\includegraphics[scale=1.2]{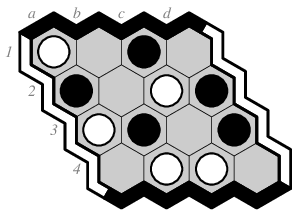}
\caption{Left: the end of a Rex game. White has joined both White sides, 
so loses. \ \ Right: a Rex puzzle by Ronald J.\ Evans. White to play and win.\cite{Gard75a}}
 \label{fig:1}  
\end{figure}

So, for positive integers $n$,
who wins Rex on $n$$\times$$n$ boards?
Using a strategy-stealing argument,
Robert O.\ Winder showed that
the first (resp.\ second) wins when $n$ is even (odd) \cite{Gard57a}.
Lagarias and Sleator further showed that, for all $n$,
each player has a strategy that can avoid defeat
until the board is completely covered \cite{LaSl99}.

Which opening (i.e.\ first) moves wins? 
Ronald J.\ Evans showed that for $n$ even,
opening in an acute corner wins \cite{Evan74}.
Hayward et al.\ further showed that, for $n$ even and at least 4,
opening in a cell that touches an acute corner cell and
one's own side also wins \cite{HTH12}.

The results mentioned so far prove the existence of winning
strategies. But how hard is it to {\em find} such strategies?
In his 1988 book Gardner commented that
``{\em 4$\times$4 [Rex] is so complex that a winning line of play
for the first player remains unknown.}
\cite{Gard75a,Gard88}.
In 2012, based on easily detected pairing strategies,
Hayward et al.\ explained how to find winning strategies
for all but one (up to symmetry) 
opening move on the 4$\times$4 board \cite{HTH12}.
\vspace*{.1cm}

In this paper, we present Solrex, an automated Rex solver that
solves arbitrary Rex positions on boards up to 6$\times$6.
With four threads,
solving the hardest 6$\times$6 opening takes under 4 hours;
solving all 18 (up to symmetry) 6$\times$6 openings takes about 7 hours.

The design of Solrex is similar to the design of the Hex program Solver.
So, Solrex searches the minimax space of gamestates using
Scalable Parallel Depth-First Proof Number Search,
the enhanced parallel version by Pawlewicz and Hayward
\cite{PawH13}
of Focussed Depth-First Proof Number Search
of Arneson, Hayward, and Henderson \cite{AHH10}.
Like Solver,
Solrex enhances the search by inferior move pruning and early win detection.
The inferior move pruning is based on Rex-specific theorems.
The win detection is based on Rex-specific virtual connections
based on pairing strategies.

In the next sections we explain
pairing strategies,
inferior cell analysis,
win detection,
the details of Solrex,
and then present experimental results.

\section{Death, pairing, capture, joining}
Roughly, a dead cell is a cell that
is useless to both players,
as it cannot contribute to 
joining either player's two sides.
Dead cells can be pruned from the Rex search tree.
Related to dead cells are captured cells,
roughly
cells that are useless to just one player 
and so can be colored for the other player.
In Hex, each player wants to capture cells;
in Rex, each player wants to force the opponent to capture cells.
In Rex, such opponent-forced capture can be brought about
by pairing strategies.
As we will see in a later section,
pairing strategies can also be used to force the opponent
to join their two sides.

Before elaborating on these ideas,
we give some basic terminology.
Let $\overline{X}$ denote the opponent of player $X$.

\begin{figure}[bht]\centering
\includegraphics[scale=1.5]{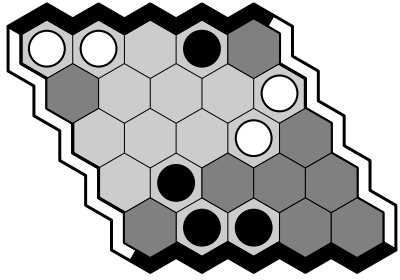}
\caption{Shaded cells are dead. All other uncolored cells are live.}
\label{fig:dead}  %
\end{figure}

For a given position,
{\em player X colors cell c}
means that player $X$ 
moves to cell $c$,
i.e.\ places a stone of her color on cell $c$.
A cell is {\em uncolored} if it is unoccupied.
To {\em X-fill} a set of cells is to $X$-color each cell in the set;
to {\em fill} a set is either to $X$-fill or $\overline{X}$-fill the set.

A {\em state} $S=P^X$ is a position $P$ 
together with the specified player $X$ to move next.
The {\em winner} of $S$ is whoever has a winning strategy from $S$.

For a position $P$ and a player $X$,
a {\em X-joinset} is a minimal set of uncolored cells
which when {\em X}-colored joins {\em X}'s two sides;
a {\em joinset} is an $X$-joinset or an $\overline{X}$-joinset;
an uncolored cell is {\em live}
if it is in a joinset, otherwise it is {\em dead};
a colored cell is {\em dead} if uncoloring it would make it dead.

For an even size subset $C$ of uncolored cells
of a position or associated state,
a {\em pairing} $\Pi$ is a partition of $C$ 
into {\em pairs}, i.e.\ subsets of size two.
For a cell $c$ in a pair $\{c,d\}$, cell $d$ is $c$'s {\em mate}.
For a state $S$, a player $Y$, and a pairing $\Pi$,
a {\em pairing strategy} is a strategy for $Y$
that guarantees that, in each terminal position reachable from $S$,
at most one cell of each pair of $\Pi$ will be $Y$-colored.

For a state $S=P^X$, {\em Last} is that player who plays
last if the game ends with all cells colored, and {\em Notlast}
is the other player, i.e.\ she who plays second-last if the game
ends with all cells uncolored.
So, Last (Notlast) is whoever plays next
if and only if
the number of uncolored cells is odd (even).
For example, for $S=P^X$ with $P$ the empty 6$\times$6 board,
Last is $\overline{X}$ and Notlast is $X$,
since $X$ plays next and $P$ has 36 uncolored cells.

\begin{theorem}
For state $S$ and pairing $\Pi$, each player
has a pairing strategy for $S$.
\end{theorem}

\begin{proof}
It suffices to follow these rules.
Proving that this is always possible is left to the reader.

First assume $Y$ is Last.
In response to $\overline{Y}$ coloring a cell in $\Pi$,
$Y$ colors the mate.
Otherwise, $Y$ colors some uncolored cell not in $\Pi$.
Next assume $Y$ is Notlast.
In response to $\overline{Y}$ coloring a cell in $\Pi$ with uncolored mate,
$Y$ colors the mate;
otherwise, $Y$ colors a cell not in $\Pi$;
otherwise (all uncolored cells are in $\Pi$, and
each pair of $\Pi$ has both or neither cell colored),
$Y$ colors any uncolored cell of $\Pi$.
\qed\end{proof}

For a player $X$ and a pairing $\Pi$ with cell set $C$
of a position $P$ or associated state $S=P^Y$,
we say $\Pi$ {\em X-captures C}
if $X$-coloring at least one cell of each pair of $\Pi$ leaves
the remaining uncolored cells of $C$ dead; and we say
$\Pi$ {\em X-joins} $P$
if $X$-coloring at least one cell of each pair of $\Pi$ 
joins $X$'s two sides.

Notice that every captured set (as defined here, i.e.\ for Rex)
comes from a pairing and so has an even number of cells,
as does every $X$-join set.

\begin{figure}[thb]\centering
\includegraphics[scale=1.2]{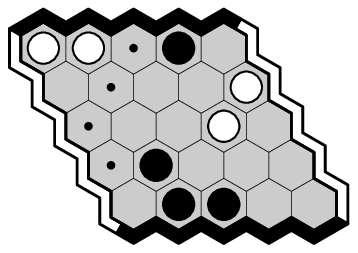}\
\includegraphics[scale=1.2]{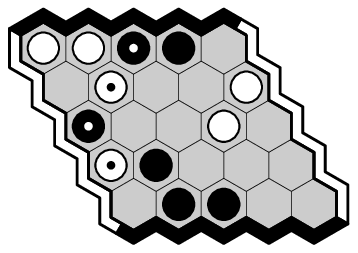}\
\includegraphics[scale=1.2]{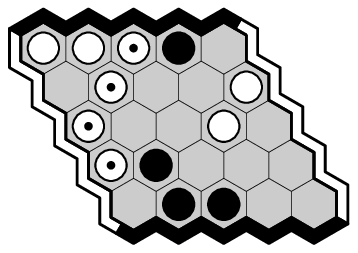}\
\caption{Left: dots show a White-captured set (the top two cells form
one pair, the bottom two form the other).
Middle: Each player has colored one cell from each pair,
and the two Black cells are dead.
Right: original position after filling White-captured cells.}
\label{fig:captured}  %
\end{figure}

\section{Inferior cell pruning and early win
detection}\label{s:prune}
We now present the Rex theorems that allow
our solver to prune inferior moves and detect wins early.

For a position $P$, a player $X$, and a set of cells $C$,
$P+C_X$ is the position obtained from $P$ by $X$-coloring all cells of $C$,
and $P-C$ is the position obtained from $P$ by uncoloring all 
colored cells of $C$. For clarity, we may also write $P-C_X$ in this case where X is the 
player who originally controlled all the cells of $C$.
Similarly, for a state $S=P^Y$, where $Y=X$ or $\overline{X}$,
$S+C_X$ is the state $(P+C_X)^Y$.
Also, in this context, when $C$ has only one cell $c$,
we will sometimes write $c_X$ instead of $\{c\}_X$.

For states $S$ and $T$ and player $X$,
we write $S \ge_X T$ if $X$ wins $T$ whenever $X$ wins $S$,
and we write $S \equiv T$ if the winner of $S$ is the winner of $T$,
i.e.\ if $S\ge_X T$ and $T \ge_X S$ for either player $X$.

An $X$-strategy is a strategy for player $X$.

\begin{theorem}\label{thm:evenfill}
For an even size set $C$ of uncolored cells of a state $S$,
$S \ge_X S+C_X$.
\end{theorem}

\begin{proof}
Assume $\pi^+$ is a winning $X$-strategy for $S^+ = S+C_X$.
Let $\pi$ be the $X$-strategy for $S$ obtained
from $\pi^+$ by moving anywhere in $C$ whenever $\overline{X}$
moves in $C$.
For any terminal position reachable from $S$, 
the set of cells occupied by $\overline{X}$
will be a superset of the cells occupied by $\overline{X}$ in
the corresponding position reachable from $S^+$,
so $X$ wins $S$.
\qed\end{proof}

\begin{theorem}\label{thm:oppocell}
For a position $P$ with uncolored cell $c$,
$(P+c_{\overline{X}})^{\overline{Y}} \geq_{X} P^Y$.
\end{theorem}

\begin{proof}
First assume $Y =\overline{X}$.
Assume $X$ wins $S=P^{\overline{X}}$.
Then, for every possible move from $S$ by $\overline{X}$,
$X$ can win.  In particular, $X$ can win after $\overline{X}$
colors $c$.  So $X$ wins
$(P+c_{\overline{X}})^X$.

Next assume $Y=X$.
Assume $X$ wins $S=P^X$.
We want to show $X$ wins
$S'=(P+c_{\overline{X}})^{\overline{X}}$.
Let $c'$ a cell to which $\overline{X}$ moves from $S'$,
let $C=\{c,c'\}$, and let $S''$ be the resulting
state $(P+C_{\overline{X}})^X$.
$X$ wins $S$ so, 
by Theorem~\ref{thm:evenfill},
$X$ wins $S''$.
So, for every possible move from $S'$, $X$ wins.
So $X$ wins $S'$.
\qed\end{proof}


%

\begin{theorem}\label{thm:ok2pair}
For an $X$-captured set $C$ of a state $S$,
$S+C_X \ge_X S$.
\end{theorem}

\begin{proof}
Assume $\overline{X}$ wins $S^+=S+C_X$ with strategy $\pi^+$. 
We want to show that $\overline{X}$ wins $S$.
Let $\Pi$ be an $X$-capture pairing for $C$, and
modify $\pi^+$ by adding to it the $\Pi$ pairing strategy for $\overline{X}$.

Let $Z$ be a terminal state reachable from $S$ by following $\pi$.
Assume by way of contradiction that $Z$ has an
$\overline{X}$-colored set of cells joining $\overline{X}$'s two sides.
If such a set $Q^*$ exists, then such a set $Q$ exists in which
no cell is in $C$.
(On $C$ $\overline{X}$ follows a $\Pi$ pairing, so in $Z$
at most one cell of each pair of $\Pi$ is $\overline{X}$-colored.
Now $X$-color any uncolored cells of $C$. Now at least one cell
of each pair is $X$-colored, and $C$ is $X$-captured,
so each $\overline{X}$-colored cell of $C$ is dead,
and these cells can be removed one at a time from $Q^*$ while still
leaving a set of cells that joins $\overline{X}$'s two sides.
Thus we have our set $Q$.)
But then the corresponding state $Z^+$ reachable from $S^+$
by following $\pi^+$ has the same set $Q$, contradicting the fact that
$\overline{X}$ wins $S^+$.
\qed\end{proof}

\begin{corollary}\label{cor:capfill}
For an $X$-captured set $C$ of a state $S$, $S \equiv S+C_X$.
\end{corollary}

\begin{proof}
By Theorem~\ref{thm:ok2pair}
and Theorem~\ref{thm:evenfill}.
\qed\end{proof}

\begin{theorem}
For a player $X$ and a position $P$ with uncolored dead cell $d$,
$(P+d_X)^{\overline{X}} \geq_X P^X$.
A move to a dead cell is at least as good as any other move.
\end{theorem}

\begin{proof}
Coloring a dead cell is equivalent to opponent-coloring the cell.
So this theorem follows by Theorem~\ref{thm:oppocell}.
\qed\end{proof}

\begin{theorem}
For a position $P$ with uncolored cells $c,k$ with $c$ dead in $P+k_X$,
$(P+c_X)^{\overline{X}}\geq_X (P+k_X)^{\overline{X}}$.
Prefer victim to killer.
\end{theorem}

\begin{proof}
$(P+k_X)^{\overline{X}} \equiv (P+k_X+c_X)^X \geq_{X} (P+c_X)^X$.
\qed\end{proof}

\begin{theorem}
For a position $P$ with uncolored cells $c,k$ with
$c$ dead in $P+k_{\overline{X}}$,
$(P+c_X)^{\overline{X}}\geq_X (P+k_X)^{\overline{X}}$.
Prefer vulnerable to opponent killer.
\end{theorem}

\begin{proof}
Assume $k$ is a winning move for $X$ from $P^X$,
i.e.\ assume $X$ wins $S=(P+k_X)^{\overline{X}}$.
Consider any such winning strategy $\pi$.
We want to show
$c$ is also a winning move for $X$ from $P^X$,
i.e.\ that $X$ wins $S'=(P+c_X)^{\overline{X}}$.

To obtain a winning $X$-strategy $\pi'$ for $S'$,
modify $\pi$ by replacing $c$ with $k$:
whenever $X$ (resp.\ $\overline{X}$) colors $c$ in $\pi$, 
$X$ ($\overline{X}$) colors $k$ in $\pi'$.
In $P$, $\overline{X}$-coloring $k$ kills $c$:
so in $P$, if some $X$-joinset $J$ contains $c$,
then $J$ must also contain $k$.
But a continuation of $\pi'$ has both $k$ and $c$ $X$-colored
if and only if the corresponding continuation of $\pi$ has
them both $X$-colored.
So, since $X$ wins $S$ following $\pi$,
$X$ wins $S'$ following $\pi'$.
\qed\end{proof}

\begin{theorem}
For a position $P$ with uncolored cell $d$ and 
set $C$ that is $X$-captured in $(P+d_x)^{\overline{X}}$,
for all $c\in C$, 
$(P+c_x)^{\overline{X}} \geq_X (P+d_x)^{\overline{X}}$.
Prefer capturee to capturer.
\end{theorem}

\begin{proof}
$(P+c_X)^{\overline{X}} \geq_X 
(P+C_X+d_X)^{\overline{X}} \equiv
(P+d_X)^{\overline{X}}$.
\qed\end{proof}


Our next results concern {\em mutual fillin},
namely when there are two cells {\em a,b} such that
$X$-coloring $a$  $\overline{X}$-captures $b$ and $\overline{X}$-coloring $b$ 
$X$-captures  $a$. 

\begin{theorem}\label{thm:mufill}
Let $P$ be a position with sets $A,B$
containing cells $a,b$ respectively,
such that $A$ is $X$-captured in $(P+b_{\overline{X}})$, 
and $B$ is $\overline{X}$-captured in $(P+a_X)$.
Then $P\equiv P+a_X+b_{\overline{X}}$.
\end{theorem}

\begin{proof}
By if necessary relabelling $\{X,a,A\}$ and $\{\overline{X},b,B\}$,
we can assume $X$ plays next.
We claim that $a$ $X$-dominates each cell in $A+B$.
Before proving the claim, observe that it implies the theorem,
since after $X$ colors $a$, all of $B$ is $Y$-captured,
so $Y$ can then color any cell of $B$, in particular, $b$.

To prove the claim,
consider a strategy that $X$-captures $A$ in $P+b_{\overline{X}}$.
Now, for all $\alpha$ in $A+B$,

\begin{align*}
(P+{\alpha}_X)^{\overline{X}} 
  & \leq_X 
  & (P+b_{\overline{X}})^{\overline{X}}  
  & \text{~ ~ ~ ~ (Theorem~\ref{thm:oppocell} twice: remove $\alpha_X$, add
  $b_{\overline{X}}$)} \\
  & \equiv 
  & (P+A_X+b_{\overline{X}})^{\overline{X}}
  & \text{~ ~ ~ ~ (capture)} \\
  & \leq_X
  & (P+a_X+B_{\overline{X}})^{\overline{X}}
  & \text{~ ~ ~ ~ (Theorem~\ref{thm:oppocell}, repeatedly for $X$ and then $\overline{X}$)} \\
  & \leq_X
  & (P+a_X)^{\overline{X}}
  & \text{~ ~ ~ ~ (capture)}\\
\end{align*}
So the claim holds, and so the theorem.
\qed\end{proof}

\begin{theorem}
Let $c$ be any $X$-colored cell
in a position $P$ as described in Theorem~\ref{thm:mufill}.
Then $(P-c+a_X)^{\bar{X}} \geq_X P^{\bar{X}}$.
Prefer filled to mutual fillin creator.
\end{theorem}

\begin{proof}
Define $b'$ to be the mate of b in the $\overline{X}$-capture strategy for $B$ in $(P+a_X)$.
\begin{align*}
P^{\bar{X}}
&\equiv 
&(P+a_X+b_{\bar{X}})^{\bar{X}}
&\text{~ ~ ~ ~ (Theorem~\ref{thm:mufill})}\\
&\equiv 
&(P+a_X+B_X-b'_{\bar{X}})^{\bar{X}}
&\text{~ ~ ~ ~ (filling captured cells, now $b'$ dead)} \\
&\equiv 
&(P+a_X+B_X)^X
&\text{~ ~ ~ ~ (coloring $b'$)}\\
&\equiv 
&(P+a_X)^X 
&\text{~ ~ ~ ~ (capture)}\\
&\leq_X 
&(P-c+a_X)^{\bar{X}}
& \text{~ ~ ~ ~ (Theorem~\ref{thm:oppocell})}\\
\end{align*}
\qed\end{proof}


\begin{figure}[htb]\centering
\hfill\hfill\includegraphics[scale=1]{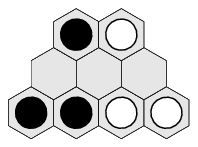}\hfill
\includegraphics[scale=1]{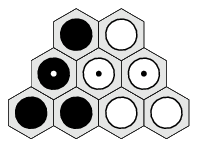}\hfill
\includegraphics[scale=1]{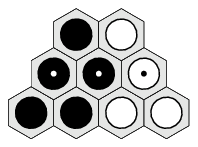}\hfill
\includegraphics[scale=1]{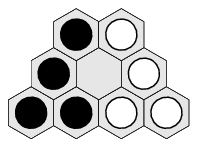}\hfill\hfill\ 
\caption{Mutual fillin.
If B colors left cell, the other two cells are W-captured. 
If W colors right cell, the other two cells are B-captured.
So we can replace first position with this.}
\label{fig:mufill}  %
\end{figure}

\begin{figure}[htb]\centering
\includegraphics[scale=1]{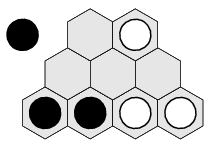}\
\includegraphics[scale=1]{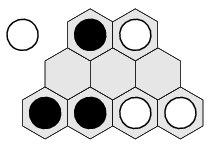}\
\includegraphics[scale=1]{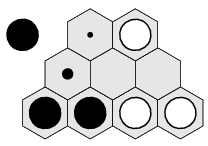}\
\includegraphics[scale=1]{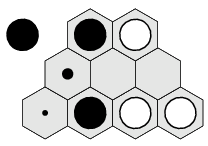}\
\includegraphics[scale=1]{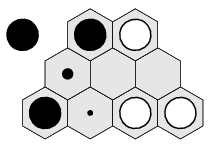}\
\caption{Mutual fillin domination. Off-board stone shows B to play.
Black move would create mutual fillin pattern.
So, for these three states, Black prefers large dot to small.}
\label{fig:caprev}  %
\end{figure}

Finally, we mention join pairing strategies.

\begin{theorem}
For a state $S=P^X$ with an $\overline{X}$-join pairing $\Pi$,
$X$ wins $P^X$.
\end{theorem}

\begin{proof}
It suffices for $X$ to follow the $\Pi$ strategy.
In each terminal state $Z$ player $X$ will have colored at most
one cell of $\Pi$.
From $Z$ obtain $Z'$ by $\overline{X}$-coloring any uncolored
cells: this will not change the winner.
But in $Z'$ at least one cell of each pair of $\Pi$ is $\overline{X}$-colored,
and $\Pi$ is an $\overline{X}$-join pairing.
So in $Z'$ $\overline{X}$'s two sides are joined,
so in $Z$ $\overline{X}$'s two sides are joined.
So $X$ wins.
\qed\end{proof}

\section{Early win detection}
For a position $P$,
a {\em X-join-pairing strategy} 
is a pairing strategy
that joins $X$'s two sides,
and an {\em X-pre-join-pairing strategy}
is an uncolored cell $k$ together with
an $X$-join-pairing strategy of $P+k_X$;
here $k$ is the {\em key} of this strategy.
The key to our algorithm is to find opponent
(pre)-join-pairing strategies.
When it is clear from context that the strategies
join a player's sides, we call these simply
{\em (pre)-pairing strategies}.

\begin{theorem}
Let $P$ be a position with an $X$-join-pairing strategy.
Then ${\overline{X}}$ wins $P^X$ and also
$P^{\overline{X}}$.
\end{theorem}

\begin{proof}
This follows from Theorem 7 in \cite{HTH12}:
$\overline{X}$ can force $X$ to follow the $X$-join-pairing strategy.
\end{proof}

\begin{theorem}
Let $P$ be a position with an $X$-pre-join-pairing
strategy and with $X=$ Last.
Then ${\overline{X}}$ wins $P^X$ and also
$P^{\overline{X}}$.
\end{theorem}

\begin{proof}
This follows from Theorem 6 in \cite{HTH12}. 
$\overline{X}$ can avoid playing the key of the pre-pairing strategy, 
forcing $X$ to eventually play it.
\end{proof}

\section{Solrex}
Solrex is based on Solhex, the 
Hex solver of the Benzene code repository \cite{Benzene}.
The challenge in developing Solrex was to identify and remove
any Hex-specific, or Rex-unnecessary, aspects of Solhex
--- e.g.\ permanently inferior cells apply to Hex but not Rex ---
and then add any Rex-necessary pieces.
E.g., it was necessary to replace the methods for
finding Hex virtual connections with methods that find
Rex (pre-) pairing strategies.

Search follows the Scalable Parallel Depth First variant
of Proof Number Search,
with the search focusing only on 
a limited number of children (as ranked by
the usual electric resistance model) at one time
\cite{PawH13}.

\begin{figure}[htb]\centering
\includegraphics[scale=.4]{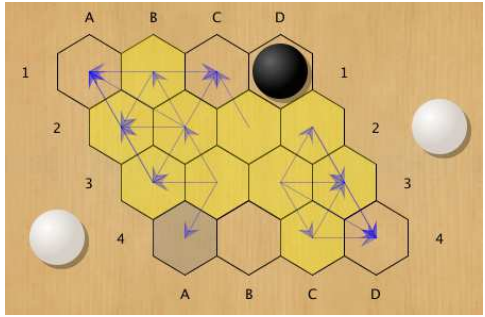}
\caption{Inferior cells of a Rex position.
Each arrow end is inferior to its arrowhead.
}
\label{fig:inf}
\end{figure}

When reaching a leaf node,
using a database of fillin and inferior cell patterns,
we apply the theorems of \S\ref{s:prune}.
We find dead cells by applying local patterns
and by searching for any empty cells whose neighbourhood
of empty cells, after stone groups have been removed
and neighbouring empty cells contracted, is a clique.
We iteratively fillin captured cells and even numbers
of dead cells until
no more fillin patterns are found.
We also apply any inferior cell domination
that comes from virtual connection decompositions 
\cite{AHH10,Hend10}.

We then look into the transposition table to see
if the resulting state win/loss value is known,
either because we previously solved, or because
of color symmetry (a state which looks the same
for each player is a win for Notlast).
Then inferior cells are pruned.
Then, using H-search \cite{Ansh00a} 
in which the or-rule
is limited to combining only 2 semi-connections,
we find (pre)-join-pairing strategies.
Then, for $X$ the player to move,
we prune each key of every $X$-pre-join-strategy.

H-search is augmented by observing
that semi-connections that overlap on a captured set
of endpoints do not conflict and so can be combined into a full
connection \cite{AHH10,Hend10}.
Notice that augmented H-search is not complete: 
some pairing strategies (e.g.\ the mirror pairing strategy
for the $n$$\times$$(n-1)$ board \cite{Gard58}) cannot be
found in this way.

Figure~\ref{fig:4x4d1} shows the start
of Solrex's solution of 1.Bd1, the
only unsolved 4$\times$4 opening from \cite{HTH12}.
First, inferior cells are found:
White b1 captures a1,a2; a2 kills a1; b2 captures a2,a3;
c2 leaves c1 dominated by b2; d2 captures d3,d4; etc.
See Figure~\ref{fig:inf}.
Only 5 White moves remain: a1,c1,a4,b4,d4.
After trying 2.Wa4, a White pre-join-pairing
strategy is found, so this loses.
Similarly, 2.Wb4 and 2.Wd4 also lose.
Now 2 White moves remain: a1,a3.
From 2.Wa1, search eventually reveals that 3.Bc1 wins
(a2 also wins).
From 2.Wc1, search reveals that 3.Ba1 wins
(b2 and d4 also win).
The deepest line in solving this position
is 1.Bd1 2.Wc1 3.Bd4 4.Wc4 5.Bb2 6.Wa3 7.Ba4 8.Wb4.

\begin{figure}[htb]\noindent
\hspace*{0cm}\includegraphics[scale=1.4]{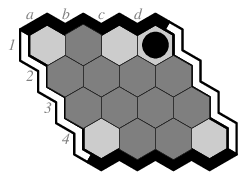}\ 
\hspace*{-.5cm}\includegraphics[scale=1.4]{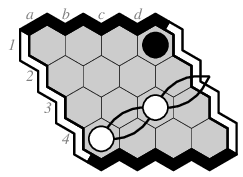}\ 
\hspace*{-.5cm}\includegraphics[scale=1.4]{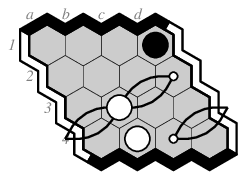}\ 
\hspace*{-.5cm}\includegraphics[scale=1.4]{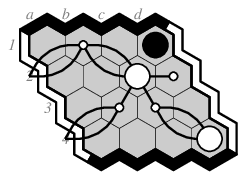}\ 
\hspace*{-.5cm}\includegraphics[scale=1.4]{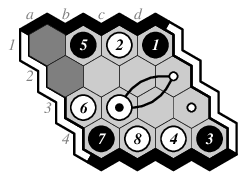}\ 
\caption{Solving 1.Bd1. Left: White inferior cells after 1.Bd1.
Then White pre-join-pair strategies after 2.Wa4,
2.Wb4, 2.Wd4. Search reveals that 2.Wa1 loses.
Search reveals that 2.Wc1 loses. So 1.Bd1 wins.
The last diagram shows the deepest line of search
and the final pre-join-pair strategy: the shaded cells
are Black-captured.}
\label{fig:4x4d1} 
\end{figure}

\section{Experiments}
We ran our experiments on Torrington,
a quad-core i7-860 2.8GHz CPU with hyper-threading, so 8 pseudo-cores.
For 5$\times$5 Rex, 
our test suite is all 24 replies to opening in the acute 
corner:\footnote{All 
opening 5$\times$5 Rex moves lose, so we picked all
possible replies to the presumably strongest opening move.}
this takes Solrex 13.2s.
For 6$\times$6 Rex,
our test suite is all 18 (up to symmetry) 1-move opening states:
this takes Solrex 25900s.
To measure speedup, we also ran the 18 1-move 6$\times$6 openings
on a single thread, taking 134635s.

To show the impact of Solrex's various features,
we ran a features knockout test on the 5$\times$5 test suite.
For features which showed negligible or negative contribution,
we ran a further knockout test on the hardest 6$\times$6 position,
1.White[d2], color-symmetric to 1.Black[e3]. 
The principle variation for this
hardest opening is shown in Figure~\ref{fig:6x6pv}.
The results are shown below.
Figure~\ref{fig:losing} shows 
all losing moves after the best opening move on 5$\times$5
(all opening 5$\times$5 moves lose),
and all losing opening moves on 6$\times$6.

Figure~\ref{fig:newpuzz} shows three new Rex puzzles
we discovered by using Solrex.
The middle puzzle was the only previously unsolved 4$\times$4 position.
The other two were found by using Solrex to search for positions with 
few winning moves.

\hfill\begin{tabular}{|c|c|}\hline
\multicolumn{2}{|c|}{\bf 5$\times$5 knockout tests} \\ \hline
version & ~ time ratio \\ \hline
all features on & 1.0 (13.9s)\\
no dead clique cutset & .97 \\
unaugmented H-search & .99 \\
no mutual fillin & 1.00 \\
no color symmetry pruning & 1.01 \\
no VC decomp & 1.06 \\
no dead fillin & 1.07 \\
no resistance move ordering ~ & 1.62 \\
no capture fillin & 2.02 \\
no inferior pruning & 2.30 \\
no H-search & 89.83 \\ \hline
\end{tabular}\hfill~

\hfill\begin{tabular}{|c|c|}\hline
\multicolumn{2}{|c|}{\bf 6$\times$6 knockout test} \\ \hline
version & ~ time ratio \\ \hline
all features on & 1.0 (13646 s)\\
unaugmented H-search &  1.10 \\ 
no color symmetry pruning & 1.13   \\ 
no dead clique cutset & 1.37 \\ 
no mutual fillin & 1.44 \\ 
no VC decomp & 1.95 \\ \hline 
\end{tabular}\hfill~

\begin{figure}[htb]\centering
\includegraphics[scale=2]{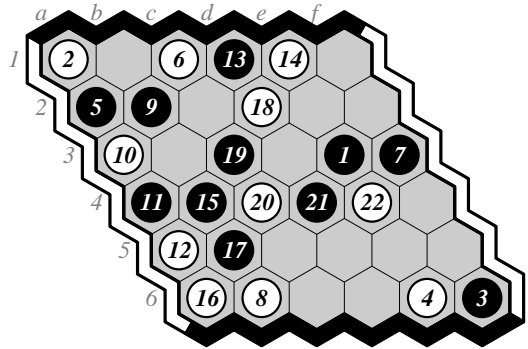}
\caption{Principle variation of 1.Black[e3], 
hardest 6$\times$6 opening. From here Black
forces White to connect with pairs \{C4,C5\} \{D6,E5\}\{F4,F5\}
and last cell D5.}
\label{fig:6x6pv}
\end{figure}

\begin{figure}[htb]\centering
\includegraphics[scale=1.5]{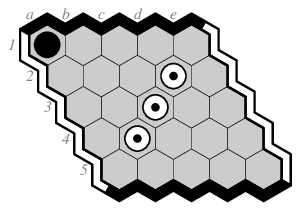}\
\includegraphics[scale=1.5]{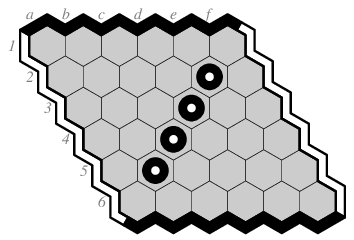}
\caption{Left: all losing replies. Right: all losing openings.}
\label{fig:losing}
\end{figure}

\begin{figure}[htb]\centering
\includegraphics[scale=1.5]{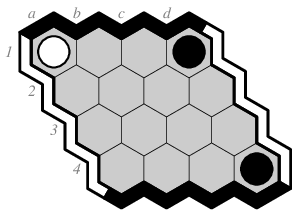}\
\includegraphics[scale=1.5]{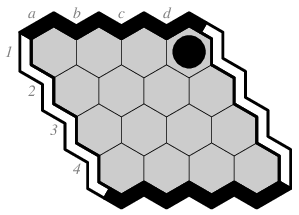}\
\includegraphics[scale=1.5]{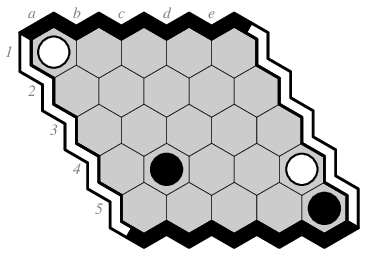}\
\caption{Three new Rex puzzles.
Left: White to play: find the only winning move. ~ Middle: White to play: find White's best move, and Black's best reply. ~ Right: White to play: find the only winning move.}
\label{fig:newpuzz}  
\end{figure}


\section{Conclusions}
All features listed in the knockout tests
contributed significantly to shortening search
time: the four features that contributed no improvement
on 5$\times$5 boards all contributed significantly
on 6$\times$6 boards. 
The effectiveness of these pruning methods --
which exploit pruning via local patterns in a search
space that grows exponentially with board size ---
explained by Henderson for Hex, is clearly
also valid for Rex \cite{Hend10}:
{\em\noindent\begin{quotation}
In almost all cases, we see that feature contributions improved
with board size. We believe this is partly because the
computational complexity of most of our algorithmic improvements is
polynomial in the board size, while the corresponding increase in
search space pruning grows exponentially.
Furthermore, as the average game length increases, more weak
moves are no longer immediately losing nor easily detectable via
previous methods, and so these features become more likely to save
significant search time.
\end{quotation}}

Of these features, by far the most critical was H-search,
which yielded a time ratio of about 90 on 5$\times$5 Rex when omitted.
The enormous time savings resulting from H-search is presumably
because our general search method does not learn to recognize
the redundant transpositions that correspond to the discovery
of a (pre-) pairing strategy. So H-search avoids 
some combinatorial explosion.

Solrex takes about 7 hours to  solve all 18 (up to symmetry) 
6$\times$6 boardstates;
by contrast, Solhex takes only 301 hours to solve all 32 (up to
symmetry) 8$\times$8 boardstates \cite{HAH09b}.
So why is Solhex faster than Solrex?

One reason is because Hex games tend to be shorter than Rex games:
in a balanced Rex game, the loser can often force the winner to
play until the board is nearly full.
Another reason is there are Hex-specific pruning features that
do not apply to Rex: for example, the only easily-found virtual
connections for Rex that we know of are pairing strategies,
and there seem to be far fewer of these than there are 
easily-found virtual connections in Hex.
Also, in Hex, if the opponent can on the next move create
more than one winning virtual connection, then the player must
make a move which interferes with each such connection or lose the
game; we know of no analogous property for Rex.

The general approach of Solhex worked well for Solrex,
so this approach might work for other games,
for example connection games such as Havannah or Twixt.

\subsubsection*{Solutions to puzzles.} 
Evans' puzzle: b1 (unique).
Three new puzzles: Left: a2 (unique). Middle: Black wins; best move for White is a1, which leaves Black with only 2 winning replies (a2, c1); all other White moves leave Black with at least 3 winning replies (e.g. c1 leaves a1, b2, d4). Right: e3 (unique).

\subsubsection*{Acknowledgments.}
We thank Jakub Pawlewicz for helpful comments.

\FloatBarrier
\bibliographystyle{plain}
\bibliography{hex}

\begin{thebibliography}{10}

\bibitem{Ansh00a}
Vadim~V. Anshelevich.
\newblock The game of {H}ex: An automatic theorem proving approach to game
  programming.
\newblock In {\em AAAI/IAAI}, pages 189--194, Menlo Park, 2000. AAAI Press /
  The MIT Press.

\bibitem{AHH10}
Broderick Arneson, Ryan~B.\ Hayward, and Philip Henderson.
\newblock Solving {H}ex: Beyond humans.
\newblock In H.\~Jaap van~den Herik, Hiroyuki Iida, and Aske Plaat, editors,
  {\em Computers and Games 2010}, volume 6515 of {\em LNCS}, pages 1--10.
  Springer, 2011.

\bibitem{Benzene}
Broderick Arneson, Philip Henderson, and Ryan~B.\ Hayward.
\newblock Benzene, 2009.
\newblock \url{http://benzene.sourceforge.net/}.

\bibitem{Evan74}
Ronald~J. Evans.
\newblock A winning opening in reverse {H}ex.
\newblock {\em Journal of Recreational Mathematics}, 7(3):189--192, 1974.

\bibitem{Gard57a}
Martin Gardner.
\newblock Mathematical games: Concerning the game of hex, which may be played
  on the tiles of the bathroom floor.
\newblock {\em Scientific {A}merican}, 197(1):145--150, June 1957.

\bibitem{Gard58}
Martin Gardner.
\newblock Mathematical games: Four mathematical diversions involving concepts
  of topology.
\newblock {\em Scientific {A}merican}, 199(4):124--129, October 1958.

\bibitem{Gard75a}
Martin Gardner.
\newblock Mathematical games: Games of strategy for two players: star nim,
  meander, dodgem, and rex.
\newblock {\em Scientific {A}merican}, 232(6):106--111, June 1975.

\bibitem{Gard88}
Martin Gardner.
\newblock {\em Hexaflexagons and Other Mathematical Diversions: The First
  {S}cientific {A}merican Book of Puzzles and Games}, chapter~8, pages 73--83.
\newblock {U}niversity of {C}hicago {P}ress, Chicago, USA, 1988.

\bibitem{HTH12}
Ryan~B.\ Hayward, Bjarne Toft, and Philip Henderson.
\newblock How to play reverse hex.
\newblock {\em Discrete Mathematics}, 312:148--156, 2012.

\bibitem{Hein42}
Piet Hein.
\newblock Vil de laere {P}olygon?
\newblock {\em Politiken}, December 1942.

\bibitem{Hend10}
Philip Henderson.
\newblock {\em Playing and solving {H}ex}.
\newblock PhD thesis, UAlberta, 2010.
\newblock \url{http://webdocs.cs.ualberta.ca/~hayward/theses/ph.pdf}.

\bibitem{HAH09b}
Philip Henderson, Broderick Arneson, and Ryan~B.\ Hayward.
\newblock Solving 8x8 {H}ex.
\newblock In Craig Boutilier, editor, {\em IJCAI}, pages 505--510, 2009.

\bibitem{LaSl99}
Jeffrey Lagarias and Daniel Sleator.
\newblock {\em The Mathemagician and Pied Puzzler: A Collection in Tribute to
  {M}artin {G}ardner, editors Elwyn Berlekamp and Tom Rodgers}, chapter~3,
  pages 237--240.
\newblock A.K. Peters, 1999.

\bibitem{PawH13}
Jakub Pawlewicz and Ryan~B.\ Hayward.
\newblock Scalable parallel dfpn search.
\newblock In {\em Computer and Games}, Springer LNCS 8427, pages 138--150,
  2013.

\end{thebibliography}
\end{document}